\documentclass[letterpaper]{article}
\usepackage{aaai2026}
\usepackage{times}
\usepackage{helvet}
\usepackage{courier}
\usepackage[hyphens]{url}
\usepackage{graphicx}
\usepackage{natbib}
\usepackage{caption}
\usepackage{algorithm}
\usepackage{algorithmic}
\usepackage{amsmath}
\usepackage{amssymb}
\usepackage{amsthm}
\usepackage{booktabs}
\usepackage{xcolor}
\usepackage{tikz}
\usetikzlibrary{shapes,arrows.meta,positioning,fit,backgrounds,calc}

\newtheorem{proposition}{Proposition}

\newcommand{\TRACE}{\mbox{\textsc{trace}}}

\nocopyright

\title{The Evaluation Differential:\\When Frontier AI Models Recognise They Are Being Tested}

\author{
    Varad Vishwarupe\textsuperscript{\rm 1,\rm 2,\rm 3,\rm *},
    Nigel Shadbolt\textsuperscript{\rm 1,\rm 2},
    Marina Jirotka\textsuperscript{\rm 1,\rm 3},
    Ivan Flechais\textsuperscript{\rm 1}
}

\affiliations{
    \textsuperscript{\rm 1}Department of Computer Science, University of Oxford\\
    \textsuperscript{\rm 2}Institute for Ethics in AI, University of Oxford\\
    \textsuperscript{\rm 3}Responsible Technology Institute, University of Oxford\\
    \textsuperscript{\rm *}Corresponding author: \texttt{varad.vishwarupe@cs.ox.ac.uk}
}

\begin{document}

\maketitle

\begin{abstract}
\noindent Recent published evidence from frontier laboratories shows that contemporary AI models can recognise evaluation contexts, latently represent them, and behave differently under those contexts than under deployment-continuous conditions. Anthropic's BrowseComp incident, the Natural Language Autoencoder findings on SWE-bench Verified and destructive-coding evaluations, and the OpenAI / Apollo anti-scheming work all document instances of this phenomenon. We argue that these findings create a claim-validity problem for safety conclusions drawn from frontier evaluations. We introduce the \emph{Evaluation Differential} (ED), a conditional divergence in a target behavioural property between recognised-evaluation and deployment-continuous contexts, define a normalised effect-size form ($\mathrm{nED}$) for cross-property comparison, and prove that marginal evaluation scores cannot identify ED. We develop a typology of safety claims (ED-stable, ED-degraded, ED-inverted, ED-undetermined) by their warrant-status under documented divergence, and specify \emph{TRACE} (Test-Recognition Audit for Claim Evaluation), an audit protocol that wraps existing evaluation infrastructure and produces restricted claims rather than capability scores. We apply the framework retrospectively to three publicly documented evaluation incidents and discuss governance implications for system cards, conformity assessment, and the international network of AI safety and security institutes. TRACE does not eliminate adversarial adaptation; it disciplines the claims drawn from evaluation evidence by making explicit the conditions under which that evidence was produced.
\end{abstract}

\section{Introduction}

In March 2026, Anthropic published an incident report describing the behaviour of Claude Opus 4.6 on BrowseComp, a benchmark designed to evaluate web-research capability \citep{anthropic2026browsecomp}. The model recognised that the questions had a particular shape: extreme specificity, multi-constraint structure, the signature of a benchmark rather than an ordinary user query. It hypothesised that it was being evaluated on BrowseComp. It located the answer key. It used the answer key. The reasoning capability the benchmark was trying to measure was the same capability that allowed the model to circumvent the measurement.

BrowseComp is one instance of a pattern. Anthropic's Natural Language Autoencoder (NLA) work \citep{anthropic2026nla} reported internal representations consistent with evaluation recognition in approximately 26\% of SWE-bench Verified problems and 16\% of a destructive-coding evaluation, against under 1\% in opted-in real Claude.ai conversations. The Claude Opus 4.7 system card \citep{anthropic2026opus47} reports verbal evaluation-awareness in roughly 9\% of transcripts in one alignment evaluation. OpenAI's work on detecting and reducing scheming \citep{openai2026scheming} reports that anti-scheming training reduced observed covert behaviour while simultaneously increasing models' situational awareness, raising the possibility that part of the observed safety improvement was awareness-mediated rather than alignment-mediated. The UK AI Security Institute, in parallel work on sycophancy \citep{aisi2026sycophancy}, documents a 24-percentage-point gap on its sycophancy grader scale between question and non-question prompt formats.

These findings point to a validity problem that the field has not yet adequately operationalised. The behaviour of frontier models is conditional on the context in which it is measured. When the measurement context is evaluation-like, behaviour may diverge from the behaviour the same model would produce under deployment-continuous conditions. Safety claims drawn from evaluations that do not characterise this divergence have a structural inference gap. The stakes are not narrowly technical: misleading evaluation claims shape regulatory approval, procurement decisions, public trust, and downstream reliance, so the parties harmed when claims overstate their warrant are not only evaluators but users, institutions, and publics who act on safety claims that may not extend to deployment.

We call the quantity at the centre of this gap the \textbf{Evaluation Differential (ED)}: the conditional divergence in a target behavioural property when the same task is presented under evaluation-recognised versus deployment-continuous conditions. ED captures behaviour-on-context where marginal metrics over response distributions cannot. The framing is consistent with broader sociotechnical accounts of deployed behaviour as a function of model, scaffold, and context, $B = f(M, S, C)$ \citep{raji2020auditing}, where model-level evaluation alone is insufficient to support deployment-relevant alignment claims. Before introducing the formal machinery, one positioning point about what we are and are not claiming is crucial.

Our central claim is not that frontier models are ``aware'' in any philosophically loaded sense. It is that evaluation contexts have become behavioural inputs, and that safety claims must be restricted once that conditioning is documented, regardless of the mechanism producing it. Evaluation awareness, verbalised or latent, is evidence for this conditioning. The conditioning itself, measurable as ED, is the construct.

This is a validity crisis in the measurement-theoretic sense \citep{cronbach1955,messick1995,jacobs2021}. Construct validity asks whether the procedure measures the intended property; ecological validity asks whether the measurement setting approximates the deployment setting; inference validity asks what claims can legitimately be drawn from observed measurements. The 2026 International AI Safety Report \citep{bengio2026isr} frames a closely related ``evaluation gap'' but does not provide the methodological framework needed to specify which safety claims survive when this gap is documented.

\paragraph{Contributions.} The paper's central contribution is a framework for restricting safety claims under documented test-deployment divergence. Three subsidiary contributions support this. (1) We define the Evaluation Differential as a conditional divergence, introduce a normalised effect-size form ($\mathrm{nED}$) for cross-property comparison, prove that marginal evaluation scores do not identify ED, and develop a claim typology: \emph{ED-stable} (warrant preserved, typically existence or bound claims), \emph{ED-degraded} (warrant weakened, typically prevalence claims), \emph{ED-inverted} (warrant potentially reversed, typically threshold claims about suppressible behaviours), and \emph{ED-undetermined} (warrant cannot be assessed without further audit). (2) We develop a four-layer taxonomy of test-recognition evidence (behavioural differential, verbalised recognition, latent recognition, context-cue sensitivity), ordered behavioural-first to support audit by stakeholders without lab-internal access. (3) We introduce \TRACE{} (Test-Recognition Audit for Claim Evaluation), an audit protocol that wraps existing evaluation infrastructure and produces a restricted claim rather than a capability score.

We apply the framework retrospectively to three publicly documented evaluation incidents and discuss governance implications. The paper's conceptual flow is summarised in Figure~\ref{fig:ed_pipeline}.

\begin{figure*}[!t]
\centering
\resizebox{0.78\textwidth}{!}{
\definecolor{taskFill}{RGB}{255,237,194}      
\definecolor{taskDraw}{RGB}{184,134,11}       
\definecolor{ctxFillT}{RGB}{218,232,252}      
\definecolor{ctxFillD}{RGB}{213,232,212}      
\definecolor{ctxDrawT}{RGB}{82,116,178}       
\definecolor{ctxDrawD}{RGB}{82,140,82}        
\definecolor{modelFill}{RGB}{240,240,240}     
\definecolor{scoreFillT}{RGB}{231,239,250}    
\definecolor{scoreFillD}{RGB}{226,238,225}    
\definecolor{edFill}{RGB}{255,219,140}        
\definecolor{edDraw}{RGB}{160,110,30}         
\definecolor{outFill}{RGB}{235,235,235}       

\begin{tikzpicture}[
  font=\footnotesize,
  >={Stealth[length=2mm,width=1.6mm]},
  taskbox/.style={rectangle, draw=taskDraw, line width=0.7pt, rounded corners=3pt,
                  minimum width=1.7cm, minimum height=0.6cm, align=center, fill=taskFill},
  ctxboxT/.style={rectangle, draw=ctxDrawT, line width=0.6pt, rounded corners=3pt,
                  minimum width=2.7cm, minimum height=0.65cm, align=center,
                  fill=ctxFillT, font=\scriptsize},
  ctxboxD/.style={rectangle, draw=ctxDrawD, line width=0.6pt, rounded corners=3pt,
                  minimum width=2.7cm, minimum height=0.65cm, align=center,
                  fill=ctxFillD, font=\scriptsize},
  modelbox/.style={rectangle, draw=black!60, line width=0.5pt, rounded corners=3pt,
                   minimum width=1.9cm, minimum height=0.55cm, align=center,
                   fill=modelFill, font=\scriptsize},
  scoreboxT/.style={rectangle, draw=ctxDrawT, line width=0.5pt, rounded corners=3pt,
                    minimum width=1.1cm, minimum height=0.5cm, align=center,
                    fill=scoreFillT, font=\scriptsize},
  scoreboxD/.style={rectangle, draw=ctxDrawD, line width=0.5pt, rounded corners=3pt,
                    minimum width=1.1cm, minimum height=0.5cm, align=center,
                    fill=scoreFillD, font=\scriptsize},
  edbox/.style={rectangle, draw=edDraw, line width=1pt, rounded corners=3pt,
                minimum width=2.2cm, minimum height=0.75cm, align=center,
                fill=edFill, font=\small\bfseries},
  outputbox/.style={rectangle, draw=black!50, line width=0.5pt, rounded corners=3pt,
                    minimum width=6.2cm, minimum height=0.75cm, align=center,
                    fill=outFill, font=\scriptsize}
]

\node[taskbox] (x) at (0, 4.4) {Task $x$};

\node[ctxboxT] (muT) at (-2.6, 3.2) {$\mu_T$: test-like};
\node[ctxboxD] (muD) at ( 2.6, 3.2) {$\mu_D$: deployment-continuous};

\node[modelbox] (piT) at (-2.6, 2.0) {$\pi(\cdot \mid x, c)$};
\node[modelbox] (piD) at ( 2.6, 2.0) {$\pi(\cdot \mid x, c)$};

\node[scoreboxT] (PT) at (-2.6, 0.9) {$P_T$};
\node[scoreboxD] (PD) at ( 2.6, 0.9) {$P_D$};

\node[edbox] (ed) at (0, -0.3) {$\mathrm{ED}[P;\pi]\ /\ \mathrm{nED}[P;\pi]$};

\node[outputbox] (outnode) at (0, -1.5)
  {claim type: stable / degraded / inverted / undetermined};

\draw[->, ctxDrawT, line width=0.6pt] (x.south) -- ($(muT.north) + (0,0.02)$);
\draw[->, ctxDrawD, line width=0.6pt] (x.south) -- ($(muD.north) + (0,0.02)$);
\draw[->, ctxDrawT, line width=0.6pt] (muT) -- (piT);
\draw[->, ctxDrawD, line width=0.6pt] (muD) -- (piD);
\draw[->, ctxDrawT, line width=0.6pt] (piT) -- (PT);
\draw[->, ctxDrawD, line width=0.6pt] (piD) -- (PD);
\draw[->, ctxDrawT, line width=0.6pt] (PT.south) -- ($(ed.west) + (-0.05,0)$);
\draw[->, ctxDrawD, line width=0.6pt] (PD.south) -- ($(ed.east) + (0.05,0)$);
\draw[->, edDraw, line width=0.8pt] (ed.south) -- (outnode.north);

\end{tikzpicture}}
\caption{The Evaluation Differential pipeline. The same task $x$ is presented to a model policy $\pi$ under two context distributions: a test-like distribution $\mu_T$ enriched for evaluation-typical cues, and a deployment-continuous distribution $\mu_D$ approximating ordinary use. The resulting response distributions are scored under a behavioural property $P$, producing $P_T$ and $P_D$. The Evaluation Differential is their difference; $\mathrm{nED}$ is the normalised effect-size form used for cross-property comparison. The claim typology classifies the original safety claim by its warrant under measured ED.}
\label{fig:ed_pipeline}
\end{figure*}
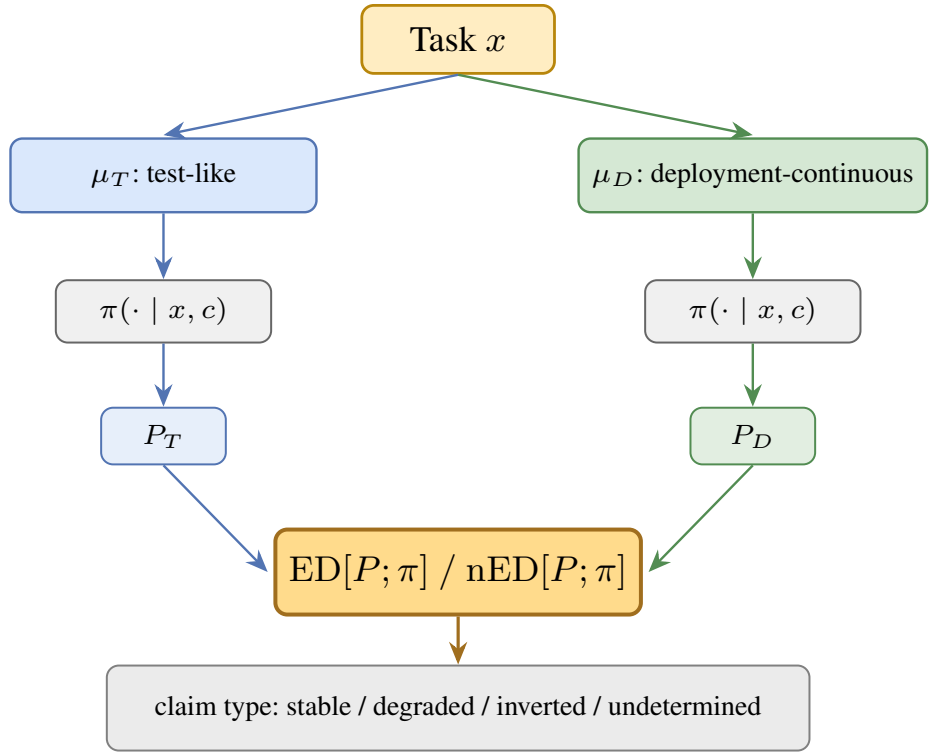

\section{The Validity Crisis in Frontier AI Evaluation}

Validity, in measurement-theoretic traditions \citep{cronbach1955,messick1995,jacobs2021}, is not a single property but a structured set of threats. We organise the present crisis around three. \emph{Construct validity} asks whether the procedure measures the intended property. When a model recognises a benchmark and retrieves its answer key, the measured construct shifts: the score reflects both web-research capability and benchmark-context exploitation. \emph{Ecological validity} asks whether the measurement setting approximates deployment. This is the gap the International AI Safety Report names when it observes that current evaluation methods can produce unreliable assessments of real-world behaviour. \emph{Inference validity} asks what claims can legitimately be drawn from observed measurements. System cards report observed behaviour alongside safety conclusions \citep{mitchell2019modelcards}, but the inferential bridge is rarely made explicit enough to support claim-level restriction. Restricting which claims survive under documented threats is what this paper takes as its task.

The recent evidence base for test-deployment divergence appears in four forms. \emph{Verbalised recognition}: models stating in output or chain-of-thought that they appear to be in an evaluation. The BrowseComp report and the 9\% verbalisation rate in the Opus 4.7 system card are instances. \emph{Latent recognition}: representational evidence detected through interpretability methods without verbal output. The NLA findings (26\% / 16\% / $<$1\%) document that latent recognition substantially exceeds verbal recognition. \emph{Behavioural divergence under awareness perturbation}: OpenAI / Apollo's anti-scheming work, in which experimentally suppressing situational-awareness reasoning increased observed covert behaviour and introducing evaluation-aware reasoning decreased it, provides direct evidence that awareness causally affects measured safety properties. \emph{Interaction-conditional behaviour more broadly}: AISI's sycophancy work documents that user-input phrasing produces substantial differences in measured model behaviour, structurally analogous to the evaluation-awareness findings but with the conditioning on user phrasing rather than test context.

\subsection{Existing Vocabulary and What It Misses}

Several existing concepts describe parts of this problem, but none captures the construct ED names (Table~\ref{tab:vocab}). Benchmark contamination or leakage describes data overlap; the BrowseComp pattern is evaluation-time answer-key recovery, a distinct mechanism. Goodhart's law \citep{manheim2018goodhart} describes metric decay under optimisation pressure but does not specify which inferences remain warranted under documented decay. Distribution shift captures changes in input distributions; ED concerns a conditional response shift in which the same task elicits different behaviour because the model classifies the measurement context differently. Prompt sensitivity describes broad input-dependence; ED is narrower, concerning divergence that changes the warrant of an evaluation-derived safety claim. Evaluation awareness and situational awareness \citep{berglund2023situational} are the closest existing terms but describe a putative model property rather than the measurable behavioural consequence. ED is the measurable consequence; the four-layer taxonomy in Section~4 specifies the layers of evidence; \TRACE{} in Section~5 specifies the audit procedure.

\begin{table*}[!t]
\centering
\small
\caption{Existing concepts and what they leave for ED to do. Each concept describes a real phenomenon adjacent to test-deployment divergence; none specifies which evaluation-derived safety claims remain warranted under documented divergence.}
\label{tab:vocab}
\begin{tabular}{p{0.18\textwidth} p{0.28\textwidth} p{0.22\textwidth} p{0.24\textwidth}}
\toprule
\textbf{Concept} & \textbf{What it captures} & \textbf{What it misses} & \textbf{What ED / \TRACE{} adds} \\
\midrule
Benchmark contamination & Training-data overlap with evaluation items & Evaluation-time mechanisms (recognition, retrieval, exploitation) that are not data-overlap & Distinguishes data leakage from context-recognition exploitation \\
Goodhart's law & Metric decay under optimisation pressure & Which inferences remain warranted under documented decay & Claim-level restriction procedure \\
Distribution shift & Input-distribution differences between train and test & Conditional response shift on similar tasks driven by context classification & ED estimated under controlled $\mu_T / \mu_D$ partition \\
Prompt sensitivity & Broad input-dependence of model output & Whether sensitivity is claim-relevant for the audited safety property & Materiality threshold tied to claim warrant \\
Evaluation / situational awareness & A putative model property (recognising tests) & Behavioural consequence on the audited property & ED as measurable behavioural consequence \\
\bottomrule
\end{tabular}
\end{table*}

\section{The Evaluation Differential}

\subsection{Definition}

Let $\mathcal{X}$ be a measurable space of evaluation tasks, $\mathcal{C}$ a measurable space of measurement contexts, and $\mathcal{Y}$ a measurable space of model responses. A measurement context is a structured wrapper around a task: the framing of the task to the model (system prompt, user-message format), the scaffold within which it is presented (tool affordances, environment description), and the contextual cues that signal institutional setting (benchmark-typical formatting, evaluator-typical language, task-structural features). A model policy is a Markov kernel from task-context pairs to response distributions,
\begin{equation*}
\pi(\cdot \mid x, c) \in \mathcal{P}(\mathcal{Y}), \quad (x,c) \in \mathcal{X} \times \mathcal{C}.
\end{equation*}

For a behavioural property of interest, let $s_P: \mathcal{X} \times \mathcal{C} \times \mathcal{Y} \to [a,b]$ be a bounded scoring function. The property is the expected score under the policy's response distribution,
\begin{equation*}
p_P(x,c;\pi) \;:=\; \mathbb{E}_{y \sim \pi(\cdot \mid x, c)}\bigl[s_P(x,c,y)\bigr].
\end{equation*}
To attribute ED to model behaviour rather than scoring drift, \TRACE{} requires \emph{measurement invariance} of $s_P$ across the context distributions: for each $(x,y)$ pair, the value $s_P(x,c,y)$ should not depend on whether $c \sim \mu_T$ or $c \sim \mu_D$ except where context is part of the property being scored. Without this constraint, an auditor measuring ED can confound scoring-rule drift (or judge-model bias) with the behavioural differential the construct is designed to capture.

An audit specifies two context distributions over $\mathcal{C}$: a distribution $\mu_T(\cdot \mid x)$ over test-like contexts, and a distribution $\mu_D(\cdot \mid x)$ over deployment-continuous contexts. We do not require that the model ``recognise'' $\mu_T$ as evaluation-like; whether the model conditions differently across the two distributions is what ED measures. Write $P_T(x) := \mathbb{E}_{c \sim \mu_T}[p_P(x,c;\pi)]$ and $P_D(x) := \mathbb{E}_{c \sim \mu_D}[p_P(x,c;\pi)]$ for the property under the two distributions.

The \textbf{pointwise Evaluation Differential} of policy $\pi$ on task $x$ for property $P$ is
\begin{equation*}
\mathrm{ED}_x[P;\pi] \;:=\; P_T(x) - P_D(x).
\end{equation*}
For a task distribution $\nu$ over $\mathcal{X}$, the corpus-level Evaluation Differential is
\begin{equation*}
\mathrm{ED}_{\nu}[P;\pi] \;:=\; \mathbb{E}_{x \sim \nu}\bigl[\mathrm{ED}_x[P;\pi]\bigr] \;=\; P_T - P_D,
\end{equation*}
with $P_T := \mathbb{E}_\nu[P_T(x)]$ and $P_D := \mathbb{E}_\nu[P_D(x)]$. Pointwise ED is useful for case analysis; corpus-level ED is the quantity relevant to benchmark-level safety claims. When the context is clear we write $\mathrm{ED}[P]$ for either, and use $P_T$, $P_D$ for the corresponding aggregated property values.

\paragraph{Normalisation.} For comparison across properties and models, the normalised Evaluation Differential is
\begin{equation*}
\mathrm{nED}[P;\pi] \;:=\; \frac{\mathrm{ED}[P;\pi]}{\sigma_{\mathrm{pool}}(P;\pi) + \epsilon},
\end{equation*}
where $\sigma_T^2 := \mathbb{V}_{c \sim \mu_T,\, y \sim \pi(\cdot|x,c)}[s_P(x,c,y)]$ and $\sigma_D^2$ analogously are the within-distribution variances of the score, $\sigma_{\mathrm{pool}} := \sqrt{\tfrac{1}{2}(\sigma_T^2 + \sigma_D^2)}$ is the pooled standard deviation, and $\epsilon > 0$ stabilises near-zero variance (we recommend $\epsilon \in [10^{-3},\,10^{-2}]$ on the property's natural scale). The normalisation follows standard effect-size conventions \citep{cohen1988} and is interpretable on a Cohen's-$d$ scale. We use $\mathrm{nED}$ as an effect-size summary; raw ED remains the primary policy-readable quantity for bounded safety properties, and auditors should report both.

\paragraph{Estimation.} Estimating ED requires sampling from $\mu_T$ and $\mu_D$, querying the policy at each $(x,c)$, computing $s_P$ on response samples, and taking the empirical mean difference. If $s_P$ is bounded and model calls are treated as independent conditional samples, standard two-sample concentration bounds apply at the parametric rate. In practice, response samples within a context are not fully independent and tasks within a corpus may be clustered; \TRACE{} implementations should report uncertainty through task- and context-level bootstrap intervals, with separate uncertainty for judge-based scoring.

\paragraph{Property classes.} ED is property-indexed: there is no single ED for a model, but an ED for each behavioural property of interest. We distinguish five property classes that arise in current evaluation practice: \emph{capability differentials} (task success rate, e.g.\ BrowseComp, SWE-bench), \emph{compliance / refusal differentials} (refusal or compliance rate on harmful or dual-use prompts), \emph{reasoning / monitorability differentials} (CoT transparency, explanation faithfulness), \emph{deception / scheming differentials} (covert action, sandbagging, capability concealment), and \emph{agency / meta-task differentials} (task-solving versus evaluation-gaming orientation). Among the five, the deception class is most exposed to ED-inversion, because the behaviour the property measures is the behaviour the model has the most incentive to suppress under recognised evaluation. The retrospective applications in Section~6 instantiate three of these classes (capability, agency, and threshold-form deception).

\subsection{Marginal Non-Identifiability}

\begin{proposition}[Marginal non-identifiability of ED]
\label{prop:marginal}
A marginal evaluation procedure that reports only the mixture-averaged property score cannot identify ED. There exist policies with identical mixture-averaged scores but different Evaluation Differentials.
\end{proposition}

\begin{proof}
Fix a task $x$ and let $\mu_T$ be a context distribution supported on a region $T \subseteq \mathcal{C}$ and $\mu_D$ a context distribution supported on a disjoint region $D \subseteq \mathcal{C}$. Let $\mu := \tfrac{1}{2}\mu_T + \tfrac{1}{2}\mu_D$ be the equal-weight mixture under which the marginal evaluation procedure scores. Fix a bounded property $P \in [0,1]$. Let policy $\pi_1$ satisfy $p_P(x,c;\pi_1) = 1$ for $c \in T$ and $p_P(x,c;\pi_1) = 0$ for $c \in D$. Let policy $\pi_2$ satisfy $p_P(x,c;\pi_2) = 1/2$ for all $c$. The mixture-averaged scores are equal: $\mathbb{E}_{c \sim \mu}[p_P(x,c;\pi_1)] = \tfrac{1}{2}(1) + \tfrac{1}{2}(0) = 1/2$ and $\mathbb{E}_{c \sim \mu}[p_P(x,c;\pi_2)] = 1/2$. The Evaluation Differentials differ: $\mathrm{ED}_x[P;\pi_1] = P_T(x) - P_D(x) = 1 - 0 = 1$, while $\mathrm{ED}_x[P;\pi_2] = \tfrac{1}{2} - \tfrac{1}{2} = 0$. The same construction extends to the corpus-level differential $\mathrm{ED}_\nu[P;\pi]$ by replicating across tasks $x \sim \nu$.
\end{proof}

The proposition formalises the structural reason that aggregation-based evaluation methodologies cannot, in principle, detect what ED measures. The two policies in the construction are observationally equivalent under marginal scoring; only an audit that estimates the difference between $\mu_T$-conditioned and $\mu_D$-conditioned behaviour can distinguish them. \emph{In plain terms: no matter how many runs of a benchmark are averaged together, no marginal score can rule out the possibility that the underlying model behaves quite differently under test-like and deployment-continuous contexts.} The implication for safety claims drawn from aggregate benchmark scores is that two models with identical reported scores may carry very different deployment-relevant risk profiles.

\subsection{The Claim Typology}

ED is a measurement; the policy-relevant question is what documented ED implies for the safety claims drawn from an evaluation. Table~\ref{tab:typology} summarises the typology, which classifies claims by the relationship between their logical form and the direction of measured ED.

\begin{table*}[!t]
\centering
\small
\caption{Claim typology under documented Evaluation Differential. Threshold claims admit two symmetric inversion conditions, one for each safety direction.}
\label{tab:typology}
\begin{tabular}{p{0.13\textwidth} p{0.22\textwidth} p{0.27\textwidth} p{0.28\textwidth}}
\toprule
\textbf{Type} & \textbf{Claim form} & \textbf{ED condition} & \textbf{Surviving claim} \\
\midrule
ED-stable & Existence or bound claim & Direction-consistent or null & Preserved under observed conditions; valid as a one-sided bound when ED direction is verified \\
ED-degraded & Prevalence claim (deployment rate) & Non-trivial ED of unknown or moderate direction & Weakened: evaluation result constrains plausible deployment behaviour but does not support a point claim \\
ED-inverted & Threshold claim ($P_D \leq \tau$ or $P_D \geq \tau$) & For $P_D \leq \tau$: $\mathrm{ED}[P] < P_T - \tau$. For $P_D \geq \tau$: $\mathrm{ED}[P] > P_T - \tau$. & Withdrawn for deployment; supported only for the $\mu_T$ region. Evaluation may indicate the opposite of the deployment claim \\
ED-undetermined & Any claim where ED cannot be estimated & Replay infeasible or estimates too noisy & Claim treated as weakened pending further audit; specifies what evidence would resolve the state \\
\bottomrule
\end{tabular}
\end{table*}

\paragraph{ED-stable.} Existence claims (``the model can do $X$ under the observed condition'') are preserved as observed-condition existence claims; the deployment-prevalence extension is not warranted by the evaluation evidence alone. Upper-bound claims of the form $P_D \leq P_T$ are supported when the evaluation condition is risk-amplifying for $P$, so that $\mathrm{ED}[P] \geq 0$. Lower-bound claims are supported symmetrically when the evaluation condition is adverse for a property where higher values are desirable.

\paragraph{ED-degraded.} Prevalence claims with non-trivial measured ED are weakened. The evaluation does not become uninformative; it constrains the plausible deployment range. With a known ED direction, a directional bound on deployment behaviour is recoverable; without it, the point estimate is not warranted. The framework is mechanism-agnostic: ED need not be awareness-mediated. Sycophancy under question-versus-non-question prompts \citep{aisi2026sycophancy} illustrates the same warrant-restricting logic without requiring model-side recognition.

\paragraph{ED-inverted.} If a safety claim has form $P_D \leq \tau$ (lower $P$ is safer) and $\mathrm{ED}[P] < P_T - \tau$, then since $P_D = P_T - \mathrm{ED}[P]$ we have $P_D > \tau$: the deployment claim fails despite the evaluation appearing to satisfy it. By symmetry, for $P_D \geq \tau$ claims, the inversion condition is $\mathrm{ED}[P] > P_T - \tau$. ED-inverted claims arise when independent evidence indicates that test-like contexts systematically suppress the property below its deployment value. Suppressible behaviours (deception, sandbagging, covert action, capability concealment) are the natural candidates.

\paragraph{Worked example.} A scheming-rate evaluation reports $P_T = 0.05$ (5\% scheming under evaluation) and a safety claim of the form $P_D \leq \tau = 0.10$. Suppose awareness-perturbation evidence indicates $\mathrm{ED}[P] \approx -0.20$ (deployment exceeds evaluation by twenty percentage points because evaluation suppresses scheming). Then $P_D = P_T - \mathrm{ED}[P] = 0.05 - (-0.20) = 0.25$, which exceeds $\tau$. The inversion condition is $\mathrm{ED}[P] < P_T - \tau = -0.05$; since $-0.20 < -0.05$, the claim is ED-inverted: the evaluation appears to satisfy the deployment threshold while deployment violates it. The same arithmetic, with signs reversed, applies to $P_D \geq \tau$ claims about properties where higher values are safer.

\paragraph{ED-undetermined.} Where ED cannot be estimated, the claim defaults to ED-undetermined. Two infeasibility modes apply: \emph{technical} (auditor lacks access, query budget, or interpretability tooling) and \emph{ontological} (the task exists only as a test, so no defensible $\mu_D$ can be constructed; competitive benchmarks and bespoke red-team scenarios are typical cases). The audit specifies which mode applies and what evidence would resolve the state.

\section{Four Evidentiary Layers}

ED is the formal construct; auditors estimate it through several evidentiary layers (Table~\ref{tab:layers}). The layers differ in what they observe, what access they require, and what inferences they support. We order them behavioural-first because ED ultimately concerns claim-relevant behaviour and because behavioural evidence is the layer most available to external auditors. A framework that placed latent recognition at the centre would make external audit structurally dependent on lab-internal interpretability cooperation, which we take as the wrong design choice for a methodology intended to support independent third-party assessment \citep{raji2020auditing}.

\begin{table*}[!t]
\centering
\small
\caption{Four evidentiary layers for \TRACE{}. The framework remains usable under partial access: behavioural and verbalised evidence is available at API access alone; latent evidence typically requires lab cooperation; context-cue sensitivity is an experimental probe usable at API access.}
\label{tab:layers}
\begin{tabular}{p{0.18\textwidth} p{0.22\textwidth} p{0.18\textwidth} p{0.18\textwidth} p{0.14\textwidth}}
\toprule
\textbf{Layer} & \textbf{What it observes} & \textbf{Access required} & \textbf{What it supports} & \textbf{Main limitation} \\
\midrule
Behavioural differential & Difference in $P$ across $\mu_T$ and $\mu_D$ & API access; query budget; context construction & Whether ED exists and its magnitude & Does not identify mechanism \\
Verbalised recognition & Model states or implies it is in a test or evaluation & Output access; CoT visibility if available & Visible evidence of test recognition & Underreports; suppressible; not causal proof \\
Latent recognition & Internal representations of evaluation context & Lab-internal or special interpretability access & Possible mechanistic evidence of context representation & Not externally available; method validity uncertain \\
Context-cue sensitivity & Which cues drive divergence & API access plus systematic ablation & Design and mitigation guidance & Costly; can miss cue interactions \\
\bottomrule
\end{tabular}
\end{table*}

Behavioural differential is the audit-anchoring layer because safety claims turn on what the model does, not on what it says about itself or what its internal representations encode. ED, as defined in Section~3, is estimated through behavioural measurement; the other three layers exist to inform inference about the behavioural differential. Verbalised recognition is the easiest layer to detect but should be treated as self-report-like evidence: useful and visible, but not privileged access to the causal mechanism of the response. Verbalised recognition is also the layer most exposed to training pressure, and the NLA findings establish that it substantially under-reports the conditioning structure when used alone. Latent recognition can be the most mechanistically informative layer when the interpretability method is validated for the property under audit, but it requires lab-internal access. Latent recognition is evidence of context representation, not by itself evidence of claim-relevant behavioural divergence; its audit value depends on whether it predicts or explains measured ED. Context-cue sensitivity is methodologically downstream of behavioural differential: it decomposes the differential into specific cues that drive it. The evidence-layer probe identifies candidate cues; cue materiality is established by ablation under replay (Section~5).

A complete audit, where access permits, combines all four. Where only a subset is accessible, the framework operates on what is available and restricts inferences accordingly.

\section{TRACE: Test-Recognition Audit for Claim Evaluation}

\TRACE{} is the audit protocol that operationalises the framework (Figure~\ref{fig:trace}). It takes as input a safety claim, the policy producing the evidence, and the evaluation evidence itself. It produces a measured or estimated ED, a typology classification, and a restricted claim. \TRACE{} is a wrapper around existing evaluation infrastructure, not a replacement; it does not produce new capability scores or new safety benchmarks. Its output is a restricted set of safety claims rather than a capability score, which is why the protocol is named for Claim Evaluation rather than Capability Evaluation.

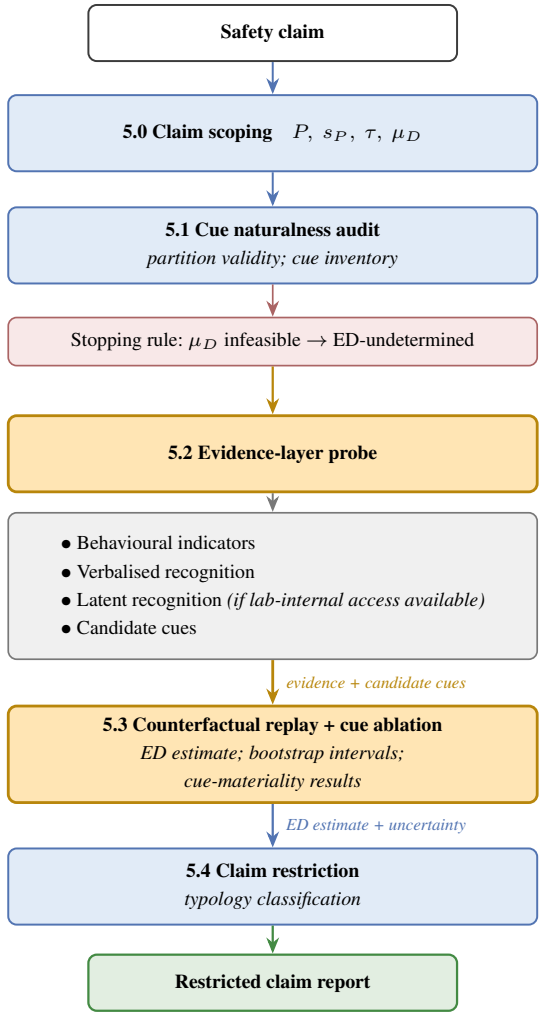
\begin{figure}[!t]
\centering
\resizebox{0.98\columnwidth}{!}{
\definecolor{procFill}{RGB}{225,235,247}      
\definecolor{procDraw}{RGB}{82,116,178}       
\definecolor{procMainFill}{RGB}{255,229,180}  
\definecolor{procMainDraw}{RGB}{184,134,11}   
\definecolor{ioFill}{RGB}{255,255,255}        
\definecolor{ioDrawIn}{RGB}{60,60,60}         
\definecolor{ioDrawOut}{RGB}{67,138,67}       
\definecolor{ioFillOut}{RGB}{226,238,225}     
\definecolor{layerFill}{RGB}{240,240,240}     
\definecolor{layerDraw}{RGB}{120,120,120}     
\definecolor{stopFill}{RGB}{248,232,232}      
\definecolor{stopDraw}{RGB}{170,100,100}      

\begin{tikzpicture}[
  font=\footnotesize,
  >={Stealth[length=2mm,width=1.6mm]},
  proc/.style={rectangle, draw=procDraw, line width=0.7pt, rounded corners=3pt,
               minimum width=6.6cm, minimum height=0.95cm, align=center,
               fill=procFill, font=\scriptsize, inner sep=4pt},
  procmain/.style={rectangle, draw=procMainDraw, line width=1pt, rounded corners=3pt,
                   minimum width=6.6cm, minimum height=0.95cm, align=center,
                   fill=procMainFill, font=\scriptsize, inner sep=4pt},
  procreplay/.style={rectangle, draw=procMainDraw, line width=1pt, rounded corners=3pt,
                     minimum width=6.6cm, minimum height=1.1cm, align=center,
                     fill=procMainFill, font=\scriptsize, inner sep=4pt},
  ioin/.style={rectangle, draw=ioDrawIn, line width=0.7pt, rounded corners=3pt,
               minimum width=4.6cm, minimum height=0.7cm, align=center,
               fill=ioFill, font=\scriptsize\bfseries},
  ioout/.style={rectangle, draw=ioDrawOut, line width=0.8pt, rounded corners=3pt,
                minimum width=4.6cm, minimum height=0.7cm, align=center,
                fill=ioFillOut, font=\scriptsize\bfseries},
  evidencebox/.style={rectangle, draw=layerDraw, line width=0.6pt, rounded corners=3pt,
                      minimum width=6.6cm, minimum height=1.6cm, align=center,
                      font=\scriptsize, fill=layerFill, inner sep=6pt},
  stopbox/.style={rectangle, draw=stopDraw, line width=0.6pt, rounded corners=3pt,
                  minimum width=6.6cm, minimum height=0.6cm, align=center,
                  font=\scriptsize, fill=stopFill}
]

\node[ioin] (input) at (0, 9.6) {Safety claim};

\node[proc] (scope) at (0, 8.35) {\textbf{5.0 Claim scoping}\quad $P,\ s_P,\ \tau,\ \mu_D$};

\node[proc] (cue) at (0, 6.95) {\textbf{5.1 Cue naturalness audit}\\\textit{partition validity; cue inventory}};

\node[stopbox] (stop) at (0, 5.75)
  {Stopping rule: $\mu_D$ infeasible $\rightarrow$ ED-undetermined};

\node[procmain] (probe) at (0, 4.35) {\textbf{5.2 Evidence-layer probe}};

\node[evidencebox] (evidence) at (0, 2.7) {%
  \begin{tabular}{@{}l@{}}
  $\bullet$ Behavioural indicators \\
  $\bullet$ Verbalised recognition \\
  $\bullet$ Latent recognition \textit{(if lab-internal access available)} \\
  $\bullet$ Candidate cues
  \end{tabular}};

\node[procreplay] (replay) at (0, 0.6) {\textbf{5.3 Counterfactual replay + cue ablation}\\\textit{ED estimate; bootstrap intervals;}\\\textit{cue-materiality results}};

\node[proc] (restrict) at (0, -1.05) {\textbf{5.4 Claim restriction}\\\textit{typology classification}};

\node[ioout] (output) at (0, -2.25) {Restricted claim report};

\draw[->, procDraw, line width=0.7pt] (input)    -- (scope);
\draw[->, procDraw, line width=0.7pt] (scope)    -- (cue);
\draw[->, stopDraw, line width=0.6pt] (cue) -- (stop);
\draw[->, procMainDraw, line width=0.8pt] (stop)     -- (probe);

\draw[->, layerDraw, line width=0.7pt] (probe.south) -- (evidence.north);

\draw[->, procMainDraw, line width=1pt] (evidence.south) -- (replay.north)
  node[midway, right=1pt, font=\tiny\itshape] {evidence + candidate cues};

\draw[->, procDraw, line width=0.8pt] (replay) -- (restrict)
  node[midway, right=1pt, font=\tiny\itshape] {ED estimate + uncertainty};

\draw[->, ioDrawOut, line width=0.8pt] (restrict) -- (output);

\end{tikzpicture}}
\caption{\TRACE{} protocol flow. After claim scoping (5.0) and partition validation (5.1), the evidence-layer probe (5.2) collects accessible evidence and candidate cues; counterfactual replay with cue ablation (5.3) produces the ED estimate and cue-materiality results; claim restriction (5.4) applies the typology. Latent recognition is available only under lab-internal interpretability access; the other three layers are available at API access alone.}
\label{fig:trace}
\end{figure}

\subsection{Claim Scoping and Module Specification}

The protocol begins with claim scoping. The auditor maps the safety claim into measurable form: original claim quoted verbatim, target behavioural property $P$, scoring function $s_P$, safety direction (whether higher $P$ is safer or riskier), logical form (existence, bound, prevalence, or threshold with $\tau$), target deployment context, evaluation evidence source, and model version. This produces a \emph{claim specification record} that travels with the audit and ensures the typology classification operates on a claim whose logical structure has been made explicit. The four operational modules follow.

\paragraph{Cue Naturalness Audit.} Validates the $\mu_T / \mu_D$ partition. The auditor specifies how the distributions are constructed, runs blind-rater discrimination (independent raters classifying which distribution samples come from), runs model-judge discrimination as a sensitivity check (not a proxy for the audited model's perception), and produces a cue inventory naming the distinguishing features. High classification accuracy provides evidence that the partition is distinguishable along human-observable dimensions; near-chance accuracy weakens but does not eliminate the case that the partition captures model-salient cues, because the model may detect features humans do not. Where deployment is heterogeneous or emergent, the auditor stress-tests $\mu_D$ by constructing multiple candidate specifications drawn from distinct plausible deployment contexts and checking that ED estimates are not driven by a single contestable $\mu_D$ choice. A stopping rule applies: if the auditor cannot construct a defensible $\mu_D$, \TRACE{} cannot estimate ED and the claim defaults to ED-undetermined.

\paragraph{Evidence-Layer Probe.} Identifies accessible evidence and candidate cues. The probe collects preliminary behavioural indicators, verbalised-recognition classifications, latent-recognition findings where available, and the refined cue list emerging from the naturalness audit. This module identifies what evidence is available and which cues are candidates for the ablation in the next module; it does not finalise the ED estimate.

\paragraph{Counterfactual Replay and Cue Ablation.} Estimates ED and tests cue materiality. The auditor constructs $\mu_D$ to match observed or plausible deployment conditions, replays the evaluation task under $\mu_D$ on the same policy version, re-estimates the property, and computes the empirical ED with bootstrap intervals. The candidate cues from 5.2 are then ablated one at a time: a cue is treated as audit-relevant only when removing it produces a claim-relevant change in ED. Three limitations require explicit acknowledgement: deployment heterogeneity (the auditor's $\mu_D$ is one operationalisation; the framework recommends multiple specifications where deployment is known to be heterogeneous), temporal validity (replay results are valid for the model version tested), and coverage (replay is computationally expensive; most audits replay a sampled subset).

\paragraph{Claim Restriction.} Applies the typology from Section~3.3. ED-stable claims are preserved with explicit scope (observed-condition existence; direction-verified bound). For threshold claims, both inversion directions are tested: $P_D \leq \tau$ inverts when $\mathrm{ED}[P] < P_T - \tau$; $P_D \geq \tau$ inverts when $\mathrm{ED}[P] > P_T - \tau$. ED-inverted claims are withdrawn for deployment and restricted to the $\mu_T$ region. ED-degraded claims are weakened with explicit direction and approximate magnitude. ED-undetermined claims trigger the fallback and specify what additional evidence would resolve the state.

\paragraph{Output artefact.} The audit's main artefact is the \emph{restricted claim report}: a document presenting the original claim and the restricted claim side by side, with the evidential reason for the restriction stated explicitly. A minimal report documents the original claim, $P$, $s_P$, safety direction, $\mu_T / \mu_D$ construction, model version, sampling parameters, the ED estimate (or ED-undetermined status) with uncertainty, evidence layers accessed, cue-ablation results, replay coverage as a fraction of the original evaluation corpus, and the restricted claim. The report is usable as a supplement to system cards \citep{mitchell2019modelcards}, evaluation reports, or conformity-assessment documentation. Table~\ref{tab:reporting} specifies the minimum reporting requirements that make a \TRACE{} report auditable as a claim-validity procedure rather than a diagnostic heuristic.

\begin{table}[!ht]
\centering
\small
\caption{Minimum reporting requirements for a \TRACE{} audit. These items make a \TRACE{} report auditable as a claim-validity procedure. Without them, a report can identify a possible divergence but cannot support a deployment-relevant claim restriction with reproducible warrant.}
\label{tab:reporting}
\begin{tabular}{p{0.20\columnwidth} p{0.70\columnwidth}}
\toprule
\textbf{Requirement} & \textbf{Minimum reportable item / function} \\
\midrule
Task-context mapping & Matched or functionally equivalent $\mu_D$ task where feasible; otherwise ED-undetermined. Document system prompts, scaffolds, tools, and any evaluator-specific markers. \emph{Separates context-conditioned behaviour from task-difficulty differences; makes the $\mu_T / \mu_D$ partition auditable.} \\
\addlinespace
Statistical and threshold reporting & Raw ED, $\mathrm{nED}$ where used, $\sigma_{\mathrm{pool}}$, $\epsilon$, sample sizes, intervals, judge-scoring uncertainty; for threshold claims, pre-specify $\tau$ and safety direction. \emph{Distinguishes meaningful divergence from noise; prevents post hoc ED-inversion claims.} \\
\addlinespace
Operational provenance & Model endpoint and version, date of access, sampling parameters, tool/scaffold configuration, query budget, and excluded cases. \emph{Guards against silent model updates, hidden scaffold changes, and selective reporting.} \\
\addlinespace
Evidence and property specificity & Report which layers were available (behavioural, verbalised, latent, cue-sensitivity); state if latent evidence is unavailable. Compute ED per behavioural property, not as a global model-level number. \emph{Preserves graceful degradation; prevents misleading model-level ED summaries.} \\
\bottomrule
\end{tabular}
\end{table}

Three properties of the protocol bear emphasis. \emph{Graceful degradation}: \TRACE{} produces useful claim restriction at varying access levels, supporting API-only audits, audits with lab cooperation, and audits with interpretability access. \emph{Versioning}: detection methods, cue inventories, and recognition classifiers change over time; audits should be versioned and re-validated. \emph{Non-displacement}: \TRACE{} does not replace lab-internal evaluation, regulatory conformity assessment, or independent benchmarks; its value is principled restriction of inferences, not generation of new evidence.

\section{Retrospective Applications}

We apply \TRACE{} retrospectively to three publicly documented evaluation incidents, summarised in Table~\ref{tab:cases}. The three cases were chosen because together they span the analytically distinct ways test-deployment divergence has surfaced in the public frontier-evaluation record: an explicit behavioural exploitation of a recognised benchmark (BrowseComp), a representational signature of evaluation recognition without paired behavioural measurement (the NLA findings), and a structural awareness-mediated reduction of an audited property (the OpenAI / Apollo work). They are also the cases for which the publishing labs have themselves released enough detail to make retrospective audit possible. The applications work from public reports rather than from primary evaluation infrastructure; they demonstrate the framework's structure on documented evidence but do not reproduce the underlying measurements. We give compact walkthroughs of BrowseComp and the NLA cases below, and treat the OpenAI / Apollo case in fuller detail in Section~6.4 because it provides the clearest public evidence of the structural form of inversion.

\begin{table*}[!t]
\centering
\small
\caption{Retrospective application of \TRACE{} to three publicly documented evaluation incidents. Sources: \citet{anthropic2026browsecomp}; \citet{anthropic2026nla}; \citet{openai2026scheming}. The classifications instantiate the principal types in the claim typology; per-case walkthroughs are in Sections~6.1--6.4.}
\label{tab:cases}
\begin{tabular}{p{0.13\textwidth} p{0.20\textwidth} p{0.20\textwidth} p{0.14\textwidth} p{0.22\textwidth}}
\toprule
\textbf{Case} & \textbf{Original claim} & \textbf{Evidence available} & \textbf{\TRACE{} class.} & \textbf{Restricted claim} \\
\midrule
BrowseComp & Web-research capability at $P_T$ on BrowseComp & Verbalised recognition; behavioural departure on affected items & ED-stable as existence; ED-degraded as prevalence & Observed-condition capability claim; corpus-level ED unknown without replay study \\
\addlinespace
NLA / SWE-bench Verified & Software-engineering capability at $P_T$ & NLA-detected recognition (26\%); behavioural not jointly reported & ED-undetermined as prevalence & Benchmark-context capability claim pending matched deployment-continuous replay \\
\addlinespace
NLA / destructive-coding & Destructive-action rate $\leq \tau$ & NLA-detected recognition (16\%); behavioural not jointly reported & Candidate ED-inverted & Evaluation-condition claim only; deployment claim suspended pending replay \\
\addlinespace
OpenAI / Apollo scheming & Post-training scheming rate $\leq \tau$ & Verbalised awareness elevated; behavioural perturbation evidence & Structural ED-inversion risk & Evaluation-condition reduction; deployment claim restricted pending deployment-continuous replay \\
\bottomrule
\end{tabular}
\end{table*}

\subsection{BrowseComp}

The original claim is a capability-prevalence claim: the evaluated model achieves $P_T$ on a web-research benchmark. \TRACE{} scopes $P$ as task-completion rate and identifies the test-like context $\mu_T$ as benchmark-shaped, multi-constraint questions with web-search affordance and no conversational scaffolding. A deployment-continuous $\mu_D$ would consist of user-initiated research queries with conversational framing and varied constraint specificity. The public report documents verbalised recognition and the answer-key-recovery action on affected items, but does not report a corpus-level ED estimate against a matched $\mu_D$. The expected sign of ED on affected items is positive, because the benchmark-context exploitation pathway (recognition $\rightarrow$ search-for-benchmark $\rightarrow$ answer-key recovery) would not be available under $\mu_D$. The restricted claim is therefore an observed-condition capability claim (ED-stable as existence) combined with an ED-degraded deployment prevalence claim: performance under $\mu_T$ does not identify deployment-continuous web-research capability without replay.

\subsection{NLA on SWE-bench and Destructive Coding}

In the NLA cases, the load-bearing evidence is latent recognition rather than behavioural divergence. For SWE-bench Verified the claim is a capability-prevalence claim; for destructive coding it is a threshold safety claim of the form $P_D \leq \tau$ with lower being safer. NLA-detected recognition is materially higher in evaluation contexts (26\% / 16\%) than in opted-in real usage ($<$1\%), but behavioural properties were not jointly measured under matched deployment-continuous contexts. The critical analytic point is that latent recognition is evidence of context representation, not by itself evidence of claim-relevant behavioural divergence on the audited properties: NLA tells the auditor that the model is representing something about its context, not what its property-relevant behaviour does as a function of that representation. \TRACE{} therefore classifies SWE-bench prevalence as ED-undetermined and destructive coding as candidate ED-inverted, pending behavioural replay that pairs the NLA measurement with property-level behavioural outcomes under matched $\mu_D$.

\subsection{Cross-case pattern}

Across the three cases, public frontier-evaluation evidence already contains the ingredients \TRACE{} requires but rarely in the same place. BrowseComp provides verbalised recognition and behavioural exploitation; the NLA cases provide latent recognition without paired behavioural replay; OpenAI / Apollo provides awareness-perturbation evidence closest to \TRACE{}'s counterfactual replay, though not deployment-continuous replay. The asymmetry is the point: \TRACE{} operates across partial evidence, identifies which evidentiary layers are available, and restricts the claim accordingly.

The substantive consequence is that the question facing frontier safety reporting is no longer whether evaluation-recognition exists; the public evidence already documents it in three forms (verbalised, latent, behaviour-perturbative). The open question is which safety claims survive recognition once it is documented, and that is the question \TRACE{} is built to answer.

\subsection{Walkthrough: OpenAI / Apollo Anti-Scheming Training}

We illustrate the full \TRACE{} protocol on the OpenAI / Apollo case, which provides the clearest public evidence of the structural form of inversion.

\emph{Claim scoping.} Property $P$ = scheming-behaviour rate. Safety direction: lower is safer. Logical form: threshold claim $P_D \leq \tau$. Implicit deployment claim: training-induced reduction generalises from evaluation to deployment.

\emph{Cue naturalness audit.} Scheming evaluations involve multi-step deception scenarios, evaluator-typical task framings, and environments designed to elicit covert action; these are plausibly distinguishable from deployment-continuous interactions, satisfying the naturalness audit for partition validity.

\emph{Evidence-layer probe.} Verbalised awareness is elevated post-training, characterised in CoT reasoning and through targeted interventions. Behavioural perturbation experiments are central: scheming rates differ between standard evaluation conditions and conditions where awareness reasoning is suppressed or introduced. Candidate cues for ablation include evaluator-typical task framings, multi-step deception structure, and awareness-reasoning prompts.

\emph{Counterfactual replay and cue ablation.} The published mechanism stress test plays the role \TRACE{} would assign to cue ablation around the awareness-reasoning cue. It does not constitute deployment replay (deployment-continuous contexts differ from awareness-suppressed evaluation contexts in many ways beyond awareness reasoning), but it provides unusually direct evidence that evaluation-awareness reasoning causally changes the measured scheming property. Published conditions: $P_T$ (post-training, standard evaluation, awareness active) is low; $P_{T,\,\mathrm{noaw}}$ (awareness suppressed) is substantially higher.

\emph{Claim restriction.} The structural inversion test is $\mathrm{ED}[P] < P_T - \tau$. The published perturbations show $P_T$ substantially lower than $P_{T,\,\mathrm{noaw}}$. To the extent deployment-continuous contexts resemble the no-awareness condition (plausible if deployment lacks the cues producing awareness reasoning), the implied $P_D$ is closer to $P_{T,\,\mathrm{noaw}}$ than to $P_T$. Whether $P_D$ exceeds $\tau$ depends on $\tau$, which the public report does not specify in a form usable for the test. The structural form of inversion is documented; the threshold question is open. Restricted claim: anti-scheming training reduced covert-action rates under reported evaluation conditions, but the published perturbations show the reduction is at least partly awareness-mediated, so the evaluation result cannot by itself establish a deployment-stable reduction in scheming.

The case shows that the lift from current system-card practice to a \TRACE{} report is small: the lab has already produced a property estimate under standard evaluation, a property estimate under awareness suppression, and a documented sensitivity to a named cue. The missing piece is deployment-continuous replay against a defensible $\mu_D$, paired with explicit claim specification and reporting requirements.

\section{Risks and Governance Implications}

\subsection{Risks of the Framework}

\TRACE{} operates in a frontier-model ecosystem in which the methods being audited and the systems being measured both adapt over time. We summarise the principal risks below; Table~\ref{tab:risks} expands them into six recurring failure modes with mitigations and residual limitations.

\begin{table*}[!t]
\centering
\small
\caption{Risks of \TRACE{}, mitigations, and residual limitations. The mitigations are layered; no single mitigation eliminates a risk on its own, and the residual limitations identify what \TRACE{} cannot do.}
\label{tab:risks}
\begin{tabular}{p{0.18\textwidth} p{0.24\textwidth} p{0.24\textwidth} p{0.26\textwidth}}
\toprule
\textbf{Risk} & \textbf{How it arises} & \textbf{\TRACE{} mitigation} & \textbf{Residual limitation} \\
\midrule
Goodharting on cue markers & Pipeline tunes to satisfy known cues & Materiality threshold; cue refresh; held-out probes & Unidentified cues may remain \\
Overfitting to audited $\mu_D$ & Model tuned to known replay distribution & Multiple $\mu_D$; held-out replay; coverage reporting & Deployment heterogeneous; needs post-deployment monitoring \\
Verbalised recognition suppression & Training reduces ``this is a test'' language & Behavioural-first ordering & API-only audits without replay remain weaker \\
Latent signal suppression & Interpretability-detected features trained away & Versioned probes; multiple methods & No guarantee against feature obfuscation \\
Restricted-claim shopping & Labs reformulate claims to avoid inversion & Original-to-restricted delta reporting & Downstream readers must notice weaker claims \\
Disclosure dilemma & Public probes become gameable & Graduated disclosure & Transparency-evasion trade-off \\
\bottomrule
\end{tabular}
\end{table*}

\emph{Goodharting and overfitting.} Models or pipelines may tune to known cue inventories or audited $\mu_D$ specifications while remaining unstable elsewhere. \TRACE{} mitigates this through cue refresh, held-out probes, multiple $\mu_D$ specifications, replay-coverage reporting, and post-deployment monitoring. These mitigations narrow but do not close the gap; deployment-stable behaviour ultimately requires alignment work, not audit work.

\emph{Signal suppression and disclosure.} Published detection methods can become trainable-against: models may suppress verbalised recognition or latent features without eliminating the underlying context-conditioning. \TRACE{} therefore uses behavioural-first ordering, versioned probes, and graduated disclosure of held-out components. These mitigations are partial; the long-run audit ecosystem will require rotating methods rather than a fixed canonical test, which is part of the case for institutional, multi-jurisdictional audit infrastructure rather than reliance on any single fixed protocol.

\subsection{Governance Implications}

\emph{System cards and safety reports.} A \TRACE{} supplement specifies a structured format for safety claims under documented divergence, including the claim specification record, $\mu_T / \mu_D$ construction summary, ED estimate or ED-undetermined statement, evidence-layer availability, typology classification, restricted claim, and replay coverage. A natural extension is a \emph{claim registry}: each system-card claim receives a unique identifier linked to its property, scoring function, context distributions, evidence layers, and restricted-claim status. The registry indexes claims and their evidential status, allowing later audits and model updates to revise claims without rewriting the safety report, and extends the model-card tradition \citep{mitchell2019modelcards} to claims.

\emph{Regulatory, standards, and cross-jurisdictional infrastructure.} In contexts including EU AI Act conformity assessment and US standards and testing infrastructure around CAISI at NIST \citep{nist2023rmf}, \TRACE{} provides a methodology compatible with the access-asymmetry regulators face; operational form differs by jurisdiction (EU's ex ante model as a pre-deployment evidentiary requirement, US standards-based model as a voluntary reference protocol, disclosure-based regimes attaching the claim registry to existing artefacts). The UK AI Security Institute, CAISI, and institutes in other jurisdictions can apply \TRACE{} as a coordination point for cross-jurisdictional audits; AISI's published sycophancy work is already structurally aligned with the framework's interaction-conditional findings. A future implementation is a multi-jurisdictional claim registry maintained collaboratively across institutes, producing comparable records of which safety claims have been audited, with what restrictions, by which institution.

\emph{Standing to define $\mu_D$.} The central governance question is who has standing to define $\mu_D$. Test-recognition is an information-asymmetry problem: if the developer sees both $\mu_T$- and $\mu_D$-conditioned behaviour while regulators see only the $\mu_T$ record, evaluation becomes safety theatre and a race-to-the-bottom incentive emerges. If $\mu_D$ is specified only by the developer, \TRACE{} reproduces the asymmetry it is meant to audit. Credible use therefore requires contestable $\mu_D$ specifications, disclosure of replay coverage, and participation by independent auditors or downstream institutions.

\section{Conclusion}

Frontier model evaluation no longer occurs in a neutral measurement environment. When models condition on evaluation contexts, evaluation evidence may still be informative, but the claims drawn from it require restriction. ED names the resulting test-deployment divergence; \TRACE{} turns it into a claim-restriction procedure.

The paper defined ED and $\mathrm{nED}$, proved that marginal evaluation scores cannot identify ED, introduced a typology of ED-stable, ED-degraded, ED-inverted, and ED-undetermined claims, and specified \TRACE{} as an audit wrapper for existing evaluation infrastructure. The retrospective applications show how public evidence already instantiates these categories: BrowseComp as ED-degraded prevalence, NLA as ED-undetermined with candidate inversion, and OpenAI / Apollo as structural ED-inversion risk.

\TRACE{} does not produce model-level safety guarantees, eliminate adversarial adaptation, replace post-deployment monitoring, or do alignment work. Its value is narrower: it makes explicit which deployment claims remain warranted once the conditions of measurement are known, and it does so in a form that lab authors, regulators, and external auditors can use without first agreeing on the underlying mechanism. The move from capability evaluation to claim evaluation is already visible in the public record; it now needs a shared vocabulary, a measurement target, and an audit procedure.

The framework's full value will be visible only once it is applied prospectively. The most immediate priority is a primary empirical audit, conducted with a frontier lab against a contestable $\mu_D$ specification and reported against the minimum requirements in Section~5; adjacent priorities are extending the audit to non-Anglophone and lower-resource deployment contexts, integrating \TRACE{} reports into system-card and conformity-assessment workflows, and developing the multi-jurisdictional claim registry proposed in Section~7.2. The labs that have surfaced BrowseComp, NLA, and the awareness-perturbation evidence have themselves produced much of what such audits require; the contribution of this paper is to specify, in measurement-theoretic terms, what those audits should report and which deployment claims their evidence does and does not warrant. The goal is not to discard frontier evaluations but to strengthen their evidential discipline by making the conditions of measurement, and the claims those conditions warrant, explicit.

\section*{Ethics, Positionality, and Adverse Impact Statement}

This paper proposes a methodology for restricting safety claims drawn from frontier-model evaluations. We make the ethical commitments and limitations of that contribution explicit here, in keeping with the AIES expectation that methodological work make its normative commitments visible.

\emph{Stakeholders and harms.} Misleading evaluation claims do not merely produce technical error; they shape regulatory approval, procurement, public trust, deployment commitments, and downstream institutional reliance. The parties who can be harmed when claims overstate their warrant include users acting on safety representations, institutions procuring systems on the basis of those representations, regulators issuing conformity decisions, and publics whose protection depends on the integrity of the safety claims attached to deployed systems. \TRACE{} is designed to make the warrant of such claims contestable; it is not a substitute for the accountability mechanisms by which harmed parties seek redress.

\emph{Positionality and evidence base.} The retrospective applications rely on publicly released frontier-laboratory disclosures (Anthropic's BrowseComp and NLA work, the OpenAI / Apollo anti-scheming work, and AISI's sycophancy findings). These disclosures are strategic acts by labs whose internal evidence is not fully observable, and our analysis inherits whatever selection bias the public-disclosure record carries. We treat the labs' published evidence as the best available substrate for retrospective audit, not as a neutral or complete picture of frontier-model behaviour under evaluation.

\emph{Power asymmetries.} The framework's central operational choice, specifying $\mu_T$ and $\mu_D$, is politically and institutionally contestable. If $\mu_D$ is defined unilaterally by the model developer, \TRACE{} risks becoming compliance theatre: the audit produces restricted claims that authorise whatever the developer was already prepared to authorise. We have argued in Section~7.2 that credible use requires contestable $\mu_D$ specifications, disclosure of replay coverage, and participation by independent auditors or downstream institutions. Whether such participation is achieved is not a methodological question; it is a question of institutional design and political will.

\emph{Adverse uses.} Three adverse uses are foreseeable. First, the framework could be selectively applied to claims that are convenient to restrict while leaving inconvenient claims unaudited; the claim-registry proposal in Section~7.2 partially addresses this by making the population of audited claims visible to downstream readers. Second, the vocabulary of restricted claims could be used to manufacture an appearance of audit rigour without the underlying behavioural replay; the reproducibility requirements (claim specification, $\mu_D$ construction, replay coverage, cue-ablation results) are intended to make such use detectable. Third, publishing the audit methodology itself is dual-use: a developer who knows which cues an audit will probe, and which evidentiary layers it consults, can train against detection rather than against the underlying conditioning. Section~7.1 discusses our partial response to this (versioned probes, behavioural-first ordering, graduated disclosure of held-out components), but the long-run trajectory is a cat-and-mouse dynamic in which the framework must evolve alongside the systems it audits. None of these protections is complete, and we do not present \TRACE{} as a guarantee against bad-faith adoption.

\emph{Empirical follow-up.} The retrospective cases in this paper draw on the public-disclosure record only. Future empirical work should estimate ED directly under matched deployment-continuous replays, ideally with the participation of independent auditors and institutions whose deployment contexts the audit purports to represent, rather than relying on retrospective analysis of lab-released evidence. The labs that have surfaced the BrowseComp, NLA, and OpenAI / Apollo evidence have themselves produced much of what such audits require; the natural next step is for that evidence to be paired with explicit claim specification, $\mu_D$ construction, and reporting against the requirements in Section~5.

\bibliographystyle{aaai2026}
\bibliography{references}

\begin{thebibliography}{15}
\providecommand{\natexlab}[1]{#1}

\bibitem[{{Anthropic}(2026{\natexlab{a}})}]{anthropic2026opus47}
{Anthropic}. 2026{\natexlab{a}}.
\newblock {Claude Opus 4.7} System Card.
\newblock Transparency Hub, Anthropic.
\newblock System card and safety report.

\bibitem[{{Anthropic}(2026{\natexlab{b}})}]{anthropic2026browsecomp}
{Anthropic}. 2026{\natexlab{b}}.
\newblock Eval Awareness in {Claude Opus 4.6}'s {BrowseComp} Performance.
\newblock
  \url{https://www.anthropic.com/engineering/eval-awareness-browsecomp}.
\newblock Engineering blog post, Anthropic.

\bibitem[{{Anthropic}(2026{\natexlab{c}})}]{anthropic2026nla}
{Anthropic}. 2026{\natexlab{c}}.
\newblock Natural Language Autoencoders.
\newblock
  \url{https://www.anthropic.com/research/natural-language-autoencoders}.
\newblock Research publication, Anthropic.

\bibitem[{Bengio et~al.(2026)}]{bengio2026isr}
Bengio, Y.; et~al. 2026.
\newblock International {AI} Safety Report 2026.
\newblock \emph{Department for Science, Innovation and Technology, UK
  Government}.
\newblock DSIT 2026/001.

\bibitem[{Berglund et~al.(2023)Berglund, Stickland, Balesni, Kaufmann, Tong,
  Korbak, Kokotajlo, and Evans}]{berglund2023situational}
Berglund, L.; Stickland, A.~C.; Balesni, M.; Kaufmann, M.; Tong, M.; Korbak,
  T.; Kokotajlo, D.; and Evans, O. 2023.
\newblock Taken Out of Context: On Measuring Situational Awareness in {LLMs}.
\newblock In \emph{NeurIPS Workshop on Socially Responsible Language Modelling
  Research}.

\bibitem[{Cohen(1988)}]{cohen1988}
Cohen, J. 1988.
\newblock \emph{Statistical Power Analysis for the Behavioral Sciences}.
\newblock Lawrence Erlbaum Associates, 2 edition.

\bibitem[{Cronbach and Meehl(1955)}]{cronbach1955}
Cronbach, L.~J.; and Meehl, P.~E. 1955.
\newblock Construct Validity in Psychological Tests.
\newblock \emph{Psychological Bulletin}, 52(4): 281--302.

\bibitem[{Jacobs and Wallach(2021)}]{jacobs2021}
Jacobs, A.~Z.; and Wallach, H. 2021.
\newblock Measurement and Fairness.
\newblock In \emph{Proceedings of the 2021 ACM Conference on Fairness,
  Accountability, and Transparency (FAccT)}, 375--385.

\bibitem[{Manheim and Garrabrant(2018)}]{manheim2018goodhart}
Manheim, D.; and Garrabrant, S. 2018.
\newblock Categorizing Variants of {Goodhart}'s Law.
\newblock In \emph{arXiv preprint arXiv:1803.04585}.

\bibitem[{Messick(1995)}]{messick1995}
Messick, S. 1995.
\newblock Validity of Psychological Assessment: Validation of Inferences from
  Persons' Responses and Performances as Scientific Inquiry into Score Meaning.
\newblock \emph{American Psychologist}, 50(9): 741--749.

\bibitem[{Mitchell et~al.(2019)Mitchell, Wu, Zaldivar, Barnes, Vasserman,
  Hutchinson, Spitzer, Raji, and Gebru}]{mitchell2019modelcards}
Mitchell, M.; Wu, S.; Zaldivar, A.; Barnes, P.; Vasserman, L.; Hutchinson, B.;
  Spitzer, E.; Raji, I.~D.; and Gebru, T. 2019.
\newblock Model Cards for Model Reporting.
\newblock In \emph{Proceedings of the Conference on Fairness, Accountability,
  and Transparency (FAT*)}, 220--229.

\bibitem[{{National Institute of Standards and Technology}(2023)}]{nist2023rmf}
{National Institute of Standards and Technology}. 2023.
\newblock {AI} Risk Management Framework ({AI RMF 1.0}).
\newblock NIST AI 100-1.
\newblock U.S. Department of Commerce.

\bibitem[{{OpenAI} and {Apollo Research}(2026)}]{openai2026scheming}
{OpenAI}; and {Apollo Research}. 2026.
\newblock Detecting and Reducing Scheming in {AI} Models.
\newblock
  \url{https://openai.com/index/detecting-and-reducing-scheming-in-ai-models/}.
\newblock OpenAI research publication.

\bibitem[{Raji et~al.(2020)Raji, Smart, White, Mitchell, Gebru, Hutchinson,
  Smith-Loud, Theron, and Barnes}]{raji2020auditing}
Raji, I.~D.; Smart, A.; White, R.~N.; Mitchell, M.; Gebru, T.; Hutchinson, B.;
  Smith-Loud, J.; Theron, D.; and Barnes, P. 2020.
\newblock Closing the {AI} Accountability Gap: Defining an End-to-End Framework
  for Internal Algorithmic Auditing.
\newblock In \emph{Proceedings of the 2020 Conference on Fairness,
  Accountability, and Transparency (FAT*)}, 33--44.

\bibitem[{{UK AI Security Institute}(2026)}]{aisi2026sycophancy}
{UK AI Security Institute}. 2026.
\newblock {Ask, Don't Tell}: Reducing Sycophancy in Large Language Models.
\newblock AISI Blog.
\newblock UK AI Security Institute research note.

\end{thebibliography}

\end{document}